\documentclass[]{article}

\usepackage{microtype}
\usepackage{graphicx}
\usepackage{subfigure}
\usepackage{booktabs}
\usepackage{hyperref}

\usepackage[accepted]{icml2023}

\usepackage{amsmath}
\usepackage{amssymb}
\usepackage{mathtools}
\usepackage{amsthm}

\usepackage[capitalize,noabbrev]{cleveref}

\usepackage{xfrac}
\usepackage{breqn}

\usepackage{mdframed}
\usepackage{thm-restate}

\newcommand{\Ebf}{\mathbf{E}}

\newcommand{\Acal}{\mathcal{A}}

\newcommand{\Hcal}{\mathcal{H}}

\newcommand{\Kcal}{\mathcal{K}}
\newcommand{\Lcal}{\mathcal{L}}

\newcommand{\Ocal}{\mathcal{O}}
\newcommand{\Pcal}{\mathcal{P}}

\newcommand{\Rcal}{\mathcal{R}}
\newcommand{\Scal}{\mathcal{S}}
\newcommand{\Tcal}{\mathcal{T}}
\newcommand{\Ucal}{\mathcal{U}}

\newcommand{\Xcal}{\mathcal{X}}

\DeclareMathOperator*{\argmin}{arg\,min}
\DeclareMathOperator*{\argmax}{arg\,max}

\newtheorem{theorem}{Theorem}[section]

\newtheorem{lemma}[theorem]{Lemma}

\newtheorem{definition}{Definition}
\newtheorem{assumption}{Assumption}
\newtheorem{axiom}{Axiom}
\newtheorem{example}{Example}

\def\emptytraj{\varepsilon}
\def\prefge{\succsim}
\def\prefgt{\succ}
\def\prefle{\precsim}
\def\preflt{\prec}
\def\defeq{\stackrel{\mathsf{def}}{=}}

\def\hists{{\cal H}}
\def\onehalf{\sfrac{1}{2}}

\def\tfrac#1#2{{\textstyle\frac{#1}{#2}}}

\def\O{{\cal O}}
\def\A{{\cal A}}
\def\S{{\cal S}}

\def\env{e}%

\icmltitlerunning{Settling the Reward Hypothesis}

\begin{document}

\twocolumn[
\icmltitle{Settling the Reward Hypothesis}

\icmlsetsymbol{equal}{*}
\begin{icmlauthorlist}
\icmlauthor{Michael Bowling}{equal,alb}
\icmlauthor{John D. Martin}{equal,alb,om}
\icmlauthor{David Abel}{dm}
\icmlauthor{Will Dabney}{dm}
\end{icmlauthorlist}
\icmlaffiliation{alb}{Amii, University of Alberta}
\icmlaffiliation{dm}{DeepMind}
\icmlaffiliation{om}{Openmind Research Institute}
\icmlcorrespondingauthor{John D. Martin}{john.martin@openmindresearch.org}

\icmlkeywords{Reinforcement Learning, Utility Theory, Reward Hypothesis, Goals}

\vskip 0.3in
]

\printAffiliationsAndNotice{\icmlEqualContribution}

\begin{abstract}
The \emph{reward hypothesis} posits that, ``all of what we mean by goals and purposes
can be well thought of as maximization of the expected value of the cumulative sum of a received
scalar signal (reward).'' We aim to fully settle this hypothesis.  This will not conclude with a simple affirmation or refutation, but rather specify completely the implicit requirements on goals and purposes under which the hypothesis holds.
\end{abstract}

\section{Introduction}
The \emph{reward hypothesis} posited by Sutton states that, ``all of what we mean by goals and purposes
can be well thought of as maximization of the expected value of the cumulative sum of a received
scalar signal (reward).''~\citep{suttonwebRLhypothesis, sutton2018reinforcement, littmanwebRLhypothesis}.
This statement takes on considerable import if one also accepts McCarthy's claim that ``Intelligence is the computational part of the ability to achieve goals in the world.''~\citep{mccarthy1998artificial}.
Together these two statements offer a sort of \emph{sufficiency} to the study of reinforcement learning (RL), whose agents learn to achieve goals through the maximization of expected future rewards.
They imply that to succeed at building AI, it is sufficient to succeed at solving RL.\footnote{Furthermore, it is possible these two claims show that RL is necessary as well as sufficient.  
If some artificial system were to be able to achieve all that we mean by goals and purposes, then such a system would have to---at least implicitly---maximize expected cumulative reward. In other words, there must be a reduction between the RL problem and the problem the system solves.}

\citet{silver2021reward} propose the related, \textit{reward-is-enough} hypothesis, which posits that
``intelligence, and all of its associated abilities, can be understood as subserving the maximization of reward.'' 
While the two hypotheses are of course deeply connected, we emphasize that our focus is on Sutton's earlier \textit{reward hypothesis}.

Sutton's original hypothesis provides an informal starting point from which to question the expressivity of reward. 
In this vein, \citet{abel2021reward} grounded the notion of ``goals and purposes'' as an ordering over policies and explored whether a Markov reward\footnote{They take a Markov reward function to be one that only depends on the most recent experience of the agent.} function could express these orderings.
They provided examples showing that a Markov reward is unable to express every such ordering. 
Their analysis also reveals how using a behavioral definition of goals can sometimes lead to unsatisfying conclusions. 
In one example (``steady-state type'' failures), an agent needs to experience unrealizable outcomes to achieve their goal.
When viewed through the lens of McCarthy's definition of intelligence, it seems that a behavioral conception of goals deflates the role that computation performs---to one of simply executing a goal's defining policy, or re-expressing the policy in a different form.
Surely intelligence involves more meaningful computation than this?

\citet{shakerinava2022utility} take a different approach.  They ground goals and purposes in preference relations over (distributions of) state-trajectories in a controlled Markov process.  In the same spirit of von Neumann-Morgenstern (vNM) utility axioms \citep{vonneumann1953theory}, they propose axioms on the preference relation (including the vNM rationality axioms) which are necessary and sufficient for the preference to be expressed with a Markov reward. Indeed, \citet{shakerinava2022utility} build on work by \citet{pitis2019rethinking} that first analyzed standard objectives of RL from the perspective of decision theory. The work of Pitis can be viewed from two complementary perspectives. First, Pitis provides a normative account for why we should embrace a state-action dependent discount factor, as developed by \citet{white17}: A fixed discount cannot capture all preferences we might consider rational. Second, Pitis presents three axioms on top of the vNM axioms that characterize the conditions under which a state-action dependent discount factor can be viewed as rational.

Our work builds off this pair of insightful approaches by starting with preferences over histories. We abandon strictly Markov processes to consider general stochastic environments and policies in line with recent work by \citet{dong2021simple}, \citet{lu2021reinforcement}, and early work on general RL \citep{lattimore2013sample,lattimore2014theory,leike2016nonparametric,majeed2021abstractions}. 
We introduce a new axiom that generalizes previous axioms from \citet{shakerinava2022utility} and accommodates the discounted reward, average reward \citep{mahadevan1996average}, and episodic settings.
Our approach posits ``goals and purposes'' as preceding environment dynamics, giving space for the agent's computational role of learning representations and behavior necessary to accomplish a goal.  
Using our new axiom along with the standard vNM rationality axioms, we provide a treatment of the reward hypothesis in both the setting that goals are the subjective desires of the agent and in the setting where goals are the objective desires of an agent designer. 
Altogether, our account does not give a simple affirmation or refutation of the reward hypothesis, but rather aims to completely specify the implicit requirements on goals and purposes under which the hypothesis holds.

\section{The Reward Hypothesis}

As we aim to settle the reward hypothesis, the first step in doing so is to formalize what it claims, and to do so in as much generality as possible. We do this by stating a series of assumptions for each of the phrases in the claim.

\subsection{Goals as Preferences}

We ground ``all of what we mean by goals and purposes'' with a binary preference relation expressing preference for one outcome over another.

The core of agent interaction is the cycle of repeatedly observing the environment and taking action to affect the environment.  Let $\O$ be a finite set of observations, and $\A$ a finite set of actions.\footnote{We assume $\Ocal$ and $\Acal$ are finite, but suspect the results generalize to the case where they are simply countable sets.}  A history is then a sequence $o_1, a_1, o_2, a_2, \ldots$ with $\emptytraj$ as the empty history of zero length.  We define the set of histories of length $n \in \mathbb{N}_{\geq 0}$ as $\hists_n \defeq (\O \times \A)^n$, and all finite length histories as $\hists \defeq \bigcup_{n=1}^{\infty} \hists_n$.  For $h \in \hists$ and transition $t\in \O \times \A$, let $t \cdot h \in \hists$ be the history with $t$ prepended to $h$.

In deterministic settings, preferences are over histories.
However, when the environment or agent behavior are stochastic, we consider distributions over histories, $\Delta(\hists)$.\footnote{We use $\Delta(\S)$ to refer to the set of all probability distributions with finite support over a countable set $\S$.} Given $A, B \in \Delta(\hists)$ and $p \in [0, 1]$, let $pA + (1-p)B \in \Delta(\hists)$ be the distribution that samples a history from $A$ with probability $p$ and $B$ with $(1-p)$. For $A \in \Delta(\hists)$ and $t \in \O \times \A$, let $t\cdot A \in \Delta(\hists)$, be the distribution where $t$ is prepended to a history sampled from $A$.

The reward hypothesis posits a ``received scalar signal''.  This could mean that the posited scalar reward signal is present in the agent's received observation, or can be computed by the agent from it.  Alternatively, it could mean that there is an additional scalar signal provided to the agent by an external observer.  We call the first setting \emph{subjective goals}, as the posited reward signal can be constructed from the agent's subjective observation, and the latter case correspondingly \emph{objective goals}.  We first develop our main result with the subjective goals setting, but will later broaden the result to objective goals.

\begin{assumption}[Subjective Goals]
``All of what we mean by goals and purposes'' can be expressed as a binary preference relation on distributions over finite histories.  For $A, B \in \Delta(\hists)$, we write $A \prefge B$ if $A$ is weakly preferred to $B$, meaning that either $A$ is strictly preferred $B$, or the two are indifferently preferred, which we denote $A\sim B$.

\label{assumption:subjective-goals}
\end{assumption}

Notice that our notion of ``goals and purposes'' make no reference to the environment.  Goals are stated as desirable histories, whereas the environment will act to constrain what histories and distributions over histories are possible.  An agent's behavior, along with the environment, then induces a distribution over histories.  Formally, given an environment $\env : \hists \rightarrow \Delta(\O)$ and policy $\pi : \hists \times \O \rightarrow \Delta(\A)$, let $D^\pi_n$ be the distribution over $\hists_n$ induced by $\env$ and $\pi$.  
\begin{multline*}
D^\pi_n \defeq \Pr\left[o_1, a_1, \ldots, o_n, a_n | \env, \pi\right] = \\
  \prod_{i=1}^n \env(o_i | o_1, a_1, \ldots a_{i-1}) \pi(a_i | o_1, a_1, \ldots, o_i).
\end{multline*}
While $D^\pi_n$ depends on the environment $\env$, we do not parameterize it as such since $\env$ typically is fixed throughout.
We assume preferences over agent policies are then consistent with the distributions over histories that they induce. 
When comparing policies, we write $\pi_1 \prefge_g \pi_2$ to mean $\pi_1$ is weakly preferred to $\pi_2$ under the goal $g$.

\begin{assumption}[Policy Preferences]
\label{assumption:policy-prefs}
We weakly prefer $\pi_1 \prefge_g \pi_2$ in $\env$ if and only if there exists $N$ such that $D^{\pi_1}_n \prefge D^{\pi_2}_n$ for all $n \ge N$. 
\end{assumption}

This notion of eventually preferring one policy's history distribution to another allows us the generality of goals and purposes that can be achieved in a defined time frame as well as those of a continuing nature. This assumption is just one simple way of handling infinite sequences, but it is not the only one. For instance, \citet{sobel1975ordinal} and \citet{pitis2019rethinking} propose alternative resolutions (horizon continuity and countable transitivity---see details in cited resources).

\subsection{Maximizing Cumulative Sums}

We now consider what the reward hypothesis says about these goals and purposes.  First, let us examine what is entailed by the ``maximization of the cumulative sum'' of a scalar reward.  This is clearly the domain of reinforcement learning, of which this problem takes a number of forms, including episodic total reward with absorbing states, infinite sum of discounted rewards, and average reward. We want to unify all of these formalisms in what could be meant by maximizing cumulative sums of rewards, and do so employing White's generalization of transition-dependent discounting~\citep{white17}.
When comparing policies under a reward $r$, we write $\pi_1\prefge_r \pi_2$.

\begin{assumption}[Cumulative Sum of Rewards]
The ``maximization of the expected value of the cumulative sum of a received scalar signal (reward)'' means that there is a reward function $r : \O \times \A \rightarrow \mathbb{R}$ and transition-dependent discount function $\gamma : \O\times\A \rightarrow [0, 1]$, such that we weakly prefer $\pi_1 \prefge_r \pi_2$ under our reward if and only if there exists $N$ such that $V^{\pi_1}_n \ge V^{\pi_2}_n$ for all $n \ge N$, where
\begin{equation}\label{eq:nstep_value}
V^\pi_n \defeq E\left[\sum_{i=1}^n\left(\prod_{j=1}^{i-1}\gamma(O_j, A_j)\right) r(O_i, A_i) \biggr| \pi, \env\right].
\end{equation}
\end{assumption}

Notice how \eqref{eq:nstep_value} simultaneously captures objectives for the average reward, discounted reward, and episodic settings.  
If $\gamma(t) = 1$ for some $o,a$ pair $t$, the objective corresponds to a typical total reward or average reward setting (even if the sum of the reward is not bounded).  
If $\gamma(t) = \gamma < 1$ is constant for all $t$, then the objective corresponds to a discounted reward objective.  If $\gamma(t) = 0$ infinitely often then the objective can correspond to an episodic setting. 
In the Appendix, we expand on the average reward case and show how the notion of the expected cumulative sum eventually being larger allows us to capture multiple kinds of optimality.

We can now state what we take the reward hypothesis to mean under all of these assumptions.

\begin{mdframed}
\vspace{1mm}
\begin{assumption}[Reward Hypothesis]
\label{asmpt:reward_hyp}
What the reward hypothesis means by ``well thought of'' is that for any preference relation on distributions of histories there exists $r$ and $\gamma$ such that $\pi_1 \prefge_g \pi_2$ under the goal $g$ iff $\pi_1 \prefge_r \pi_2$. %
\end{assumption}
\end{mdframed}

In what follows, we explore if this is true or false, or more precisely, what might be required of our preference relation for the reward hypothesis to hold.

\section{Rationality Axioms}

The vNM axioms provide necessary and sufficient conditions for a preference relation to be expressible as the expectation of some scalar-valued function of outcomes.  The ``expected value of the cumulative sum'' of rewards is central to the reward hypothesis, so we present these axioms here along with the corresponding vNM utility theorem.  In the statement of these and additional axioms, $\hists$ is any set of finite length sequences of some countable set of transitions $\Tcal$ (e.g., $\Tcal$ may be $\O \times \A$ as in the subjective goals case of \autoref{assumption:subjective-goals}).

We next state each of the four vNM axioms alongside some brief intuition.

\begin{axiom}[Completeness]
\label{axiom:completeness}
For all $A, B \in \Delta(\hists)$, $A\prefge B$ or $B \prefge A$ (or both, if $A \sim B$).
\end{axiom}

Completeness requires that the preference ordering make \textit{some} judgment about any pair of distributions. Note that the preference \textit{could} simply convey indifference: we might be equally satisfied with an apple and a banana. This is distinct from having no preference at all (\citet{chang2015value} discusses the incomparability of virtues like ``justice'' and ``mercy''). Note that there are alternative sets of axioms that simply remove completeness, as developed by \citet{aumann1962utility}.

\begin{axiom}[Transitivity]
\label{axiom:transitivity}
For all $A, B, C \in \Delta(\hists)$, if $A\prefge B \prefge C$, then $A \prefge C$.
\end{axiom}

Transitivity is relatively straightforward: no coherent goal can involve cyclical preferences.

\begin{axiom}[Independence]
\label{axiom:independence}
For all $A, B, C \in \Delta(\hists)$ and $p \in (0, 1)$, $A \prefge B$ if and only if
\[
pA + (1-p) C \prefge pB + (1-p) C
\]
\end{axiom}

Independence was historically viewed as an unlikely axiom, and one that led to skepticism about vNM's initial results \citep{machina1990expected}. However, it does convey a relatively powerful intuition, once it is unpacked.
Consider the following example. Suppose we must choose between an Apple, a Banana, or Chocolate. We are asked: (1) Do you prefer one Apple to one Banana? and (2) Consider two coins of the same weight which, when flipped can yield \{H: Apple, T: Chocolate\}, and \{H: Banana, T: Chocolate\}---which coin do you prefer to flip? Independence requires that for \textit{all} coin weights, your answer to (1) and (2) must be the same. In other words, if you truly prefer an Apple to a Banana, then the \textit{chance} of procuring the Apple and Banana should not change this preference (when the other alternatives are the same). Independence is deeply connected to many alternative objectives in RL such as risk-aversion and multi-objective RL, which we discuss in more detail shortly.

\begin{axiom}[Continuity]
\label{axiom:continuity}
For all $A, B, C \in \Delta(\hists)$ if  $A\prefge B \prefge C$, then there exists $p \in [0, 1]$ such that, 
\[
pA + (1-p) C \sim B
\]
\end{axiom}

Continuity demands the existence of a break-even point when one becomes indifferent to a mixture of outcomes that are individually more or less preferred.
At what precise coin weight does one become indifferent between a Banana and a distribution of Apple and Chocolate? 

These four axioms comprise the typical account of rational preferences under vNM's theory. Before stating the classical vNM theorem, we fix some notation. The Dirac distribution for any history $h\in\Hcal$ is $\delta_h$, and its utility is $u(h)=u(\delta_h)$.

\begin{theorem}[von Neumann-Morgenstern Utility Theorem]
A binary preference relation on $\Delta(\hists)$ satisfies axioms 1-4 if and only if there exists a utility function $u :\Delta(\hists) \rightarrow \mathbb{R}$ such that,
\begin{enumerate}
    \item $\forall A, B \in \Delta(\hists)\quad A \prefge B \Leftrightarrow u(A) \ge u(B),$
    \item $\forall (\sum_i p_i h_i) \in \Delta(\hists)\quad u(\sum_i p_i h_i) = \sum_i p_i u(h_i)$,
\end{enumerate}
and $u$ is unique up to positive affine transformations.
\label{thm:vnm}
\end{theorem}
In other words, there exists a utility function whose expectation for any distribution over histories is consistent with the preference relation. 

These rationality axioms are sufficient for our interpretation of the reward hypothesis (\autoref{asmpt:reward_hyp}) to trivially hold with no further assumptions.  Simply define the `received' reward for experiencing transition $t$ following history $h$ as the change in utility from appending transition $t$ to $h$, written as $r(t ; h) = u(h \cdot t) - u(h)$.  However, this reward function is not, in general, computable from the agent's received observation as it depends on the entire history. 
This definition of reward function may not even be computable by a bounded agent or a bounded external observer as, in general, it requires complete memory of the history.  We will need to add an additional axiom for the reward hypothesis to hold for a Markov reward; that is, a reward that is received or computable from the agent's most immediate experience.%

\section{A New Axiom: Temporal $\gamma$-Indifference}
Here we derive a new axiom that ensures the existence of Markov rewards and implicitly specifies the type of objective faced by an agent. 
We start from an observation about how preferences can still encode statements about magnitude, even without reducing them to a scalar value.

Let $A, B, C, D \in \Delta(\hists)$ with $A \prefge B$ and $C \prefge D$.  Suppose we wanted to state that $A$ is preferred over $B$ \emph{by the same amount} as $C$ is preferred over $D$.  If we could encode preferences as utilities, then we could write
\begin{align*}
u(A) - u(B) &= u(C) - u(D),\\
u(A) + u(D) &= u(B) + u(C), \\
\onehalf u(A) + \onehalf u(D) &= \onehalf u(B) + \onehalf u(C),\\
u(\onehalf A + \onehalf D) &= u(\onehalf B + \onehalf C).
\end{align*}
Notice that now we are just comparing utilities of two distributions.  We could equivalently state this entirely through the preference relation:
\[
\onehalf A + \onehalf D \sim \onehalf B + \onehalf C.
\]

Furthermore, suppose we wanted to state that $A$ is preferred to $B$ by a multiplicative factor $\alpha  > 0$ of how much $C$ is preferred to $D$.  Again, if we had an equivalent utility function we could write  
\begin{align*}
u(A) - u(B) &= \alpha(u(C) - u(D)), \\     
u(A) + \alpha u(D) &= u(B) + \alpha u(C), \\
\tfrac{1}{1+\alpha} u(A) + \tfrac{\alpha}{1+\alpha} u(D) &=
\tfrac{1}{1+\alpha} u(B) + \tfrac{\alpha}{1+\alpha} u(C), \\
u\left(\tfrac{1}{1+\alpha} A + \tfrac{\alpha}{1+\alpha} D\right) &=
u\left(\tfrac{1}{1+\alpha} B + \tfrac{\alpha}{1+\alpha} C\right), \\
\tfrac{1}{1+\alpha} A + \tfrac{\alpha}{1+\alpha} D &\sim
\tfrac{1}{1+\alpha} B + \tfrac{\alpha}{1+\alpha} C .
\end{align*}
So we can write the same concept entirely within the preference relation.

Now if we consider a transition $t\in \Tcal$ and two distributions over histories $A,B\in\Delta(\hists)$, we can use the above to state that $t\cdot A$ is preferred to $t\cdot B$ by a multiplicative factor $\gamma(t)$ of the amount $A$ is preferred to $B$:
\begin{align*}
\tfrac{1}{\gamma(t) + 1}(t \cdot A) + \tfrac{\gamma(t)}{\gamma(t) + 1} B \sim
\tfrac{1}{\gamma(t) + 1}(t \cdot B) + \tfrac{\gamma(t)}{\gamma(t) + 1} A.
\end{align*}
This brings us to what we call the Temporal $\gamma$-Indifference axiom.

\begin{axiom}[Temporal $\gamma$-Indifference]
\label{axiom:temporal_gamma_indifference}
For all $A, B \in \Delta(\hists)$ and transitions $t \in \Tcal$,
\begin{align*}
    \tfrac{1}{\gamma(t) + 1}(t \cdot A) + \tfrac{\gamma(t)}{\gamma(t) + 1} B \sim 
\tfrac{1}{\gamma(t) + 1}(t \cdot B) + \tfrac{\gamma(t)}{\gamma(t) + 1} A.
\end{align*}
\end{axiom}
Notice this axiom is parameterized by a discount function $\gamma : \Tcal \rightarrow [0, 1]$ defined on transitions $\Tcal$.

The axiom essentially requires that $t\cdot A$ is preferred to $t\cdot B$ by a multiplicative factor $\gamma(t)$ of how much $A$ is preferred to $B$.
This fact is illustrated in the following example.

\begin{example}
Suppose $\gamma(t) = 1$ for all transitions $t \in \Tcal$.  Then, the temporal indifference axiom states that for all $h_1, h_2 \in \hists$ and transitions $t \in \Tcal$, 
\[
\onehalf(t \cdot h_1) + \onehalf h_2 \sim \onehalf (t \cdot h_2) + \onehalf h_1.
\]
In other words, given any two histories to be experienced with equal probability, the agent is indifferent to which history gets prepended with a transition, regardless of the transition.  No matter which history is prepended with $t$, the transition $t$ must be experienced.  So the indifference is requiring the agent has no preference over which history is delayed, even if one history is highly preferred to the other.
\end{example}

We now state our main result. All proofs are included in the Appendix. A formalization of this theorem in Lean~4 is available at \url{https://github.com/jdmartin86/reward-hypothesis}.\\

\begin{restatable}[Markov Reward Theorem]{theorem}{markovrewardtheorem}
\label{thm:markov-reward}
A binary preference relation on $\Delta(\hists)$ satisfies Axioms 1-5 if and only if there exists a utility function $u\colon \Delta(\hists) \rightarrow \mathbb{R}$, a reward function $r : \Tcal \rightarrow \mathbb{R}$, and transition-dependent discount function $\gamma : \Tcal \rightarrow [0, 1]$, such that 
\begin{enumerate}
    \item $u(\emptytraj) = 0$,
    \item $\forall t\in \Tcal, h\in \Hcal \; u(t \cdot h) = r(t) + \gamma(t) u(h)$,
    \item $\forall A, B \in \Delta(\hists)\; A \prefge B \Leftrightarrow u(A) \ge u(B),$
    \item $\forall (\sum_i p_i h_i) \in \Delta(\hists)\; u(\sum_i p_i h_i) = \sum_i p_i u(h_i)$,
\end{enumerate}
where $r$ is unique up to a positive scale factor, and $\gamma$ is the function for which Axiom 5 is satisfied. 
\end{restatable}
 In other words, there exist a deterministic, Markov reward function such that the expected sum of rewards under a particular transition-dependent discount factor is consistent with the preference relation.
 Furthermore, we can show there exists an efficient algorithm that constructs the reward function and discount factor from a preference relation that satisfies Axioms 1--5 (See \autoref{alg:reward_design} in the Appendix for additional details.). 

Additionally, notice the form the objective takes (e.g. discounted reward, episodic total reward, average reward) is determined by the preference relation and how it satisfies the Temporal $\gamma$-Indifference axiom.
This is what ultimately dictates $\gamma(t)$.

\section{Objective Goals}

\label{sec:obj_goals}

The results thus far assume that the preferences and rewards of interest originate from the same perspective---that is, the agent doing the maximizing is the same as the agent holding the preferences. 
In practice, we often find ourselves in a different setting in which the relevant preferences originate from an \textit{agent designer} that has a desired goal or purpose in mind for a separate \textit{learning agent} to pursue (Alice and Bob in work by \citet{abel2021reward}, or the ``designer'' and ``agent'' in work by \citet{singh2009rewards}). 
In this section we adapt our previous assumptions to show how the results of \autoref{thm:markov-reward} apply more broadly---to arbitrary sequences that contain the designer's experiences. 
This is what we refer to as the \textit{objective goals} case.
Indeed, this setup includes common cases from the literature where preferences are expressed in numerous ways: as demonstrations of desired behaviors \citep{ng2000algorithms}, partial orders over a set of trajectories \citep{wirth2017survey}, or through a generic interaction process with the designer \citep{leike2018scalable}.

In the objective goals setting, the designer experiences a stream of observations $\bar{o} \in \bar{\Ocal}$.
These form a distinct process which is potentially related to the observation stream of the agent.
For instance, the designer may observe more than the agent: $\Ocal \subset \bar{\Ocal}$.
In other cases, $\bar\Ocal$ may include the agent's actions.
The designer provides the agent with a learning signal that reflects its preferences on distributions over histories $\bar{h}_t = \bar{o}_1, \bar{o}_2, \ldots, \bar{o}_t$, where each $\bar{o}_i \in \bar{\Ocal}$. 
We let $\bar{\hists}_n$ be all histories of length $n$ and $\bar{\hists} \defeq \bigcup_{n=0}^\infty \bar{\hists}_n$ be all finite length histories over designer observations.
In what follows, we suppose the designer maintains a preference relation over distributions from $\Delta(\bar{\hists})$, then we adapt our assumptions so the results of \autoref{thm:markov-reward} apply to this setting.

\begin{assumption}[Objective Goals]
``All of what we mean by goals and purposes'' can be expressed as a binary preference relation on distributions over finite histories of designer observations $\Delta(\bar{\hists})$. %
\end{assumption}

We have two choices here for defining the agent's interface to the environment.
First, the designer can provide the rewards $r$ and the discounts $\gamma$ as separate inputs to the agent.
Second, the designer can provide rewards that are already discounted,
\begin{align}\label{eq:prediscounted_rewards}
    r_i\defeq \left(\prod_{j=1}^{i-1}\gamma(\bar{o}_j)\right) r(\bar{o}_i).
\end{align}
We adopt the second view to maintain the standard agent-environment interface, but note that the former view might yield an alternative plausible account.

\begin{assumption}[Cumulative Sum of Objective Rewards]
\label{assumption:sum_of_rewards_policy_pref}
The ``maximization of the expected value of the cumulative sum of a received scalar signal (reward)'' means that there is a reward function $r : \bar\Ocal \rightarrow \mathbb{R}$ and transition-dependent discount function $\gamma : \bar\Ocal \rightarrow [0, 1]$, such that a designer prefers $\pi_1 \prefge_r \pi_2$ under the reward $r$ if and only if there exists $N$ such that  $V^{\pi_1}_n > V^{\pi_2}_n$ for all $n \ge N$, where $V^\pi_n = E\left[\sum_{i=1}^n r_i\right].$
\end{assumption}

With this we can now restate our interpretation of the reward hypothesis for the objective goals case.
What the reward hypothesis means by ``well thought of'' is that for any binary preference relation on distributions of a designer's histories, there exists an already discounted $r$, as defined in \eqref{eq:prediscounted_rewards}, such that $\pi_1 \prefge_g \pi_2$ if and only if $\pi_1 \prefge_r \pi_2$.

\section{History and Related Work}
We next discuss relevant literature from across the RL community. We discuss connections to the economics literature in the appendix.

\subsection{Utility Theory for Sequential Decision Making}

\citet{pitis2019rethinking} explored the relationship between the vNM axioms and the typical objectives of RL, with a focus on the discount factor, $\gamma$. As discussed in the introduction, Pitis provides a normative account for why we should embrace a state-action dependent discount factor, as developed by \citet{white17}: A fixed discount cannot capture all preferences we might consider rational. Second, Pitis presents three axioms on top of the vNM axioms that characterize the conditions under which a state-action dependent discount factor can be viewed as rational. These axioms bear some resemblance to aspects of our formalism: For instance, Pitis' countable transitivity axiom also addresses the issue of infinite experiences, like our \autoref{assumption:policy-prefs}.

\citet{shakerinava2022utility} build off the work of \citet{pitis2019rethinking} to study the existence of reward functions in a variety of controlled MDPs.
In their work, Shakerinava \& Ravanbakhsh formalize the notion of goals with a preference relation over outcomes generated from the given MDP.

Our work differs from \citet{shakerinava2022utility} in several ways. 
Firstly, we take into account preference relations over distributions of observation-action histories, enabling us to reason about goals in a wide range of stochastic environments, including fully-observable MDPs.
Furthermore, we formalize the reward hypothesis with a set of formal assumptions.
This applies to the discounted reward, average reward, and episodic total reward settings. 
We establish a connection between our findings and scenarios where the reward is internally generated by the agent, as well as situations involving the objective desires of an observer.
We show there exists an efficient algorithm to construct rewards using a preference relation that satisfies Axioms 1--5 (See \autoref{alg:reward_design} in the Appendix).
Moreover, our Temporal $\gamma$-Indifference axiom provably generalizes two of their axioms.
We expand on this below.

\begin{axiom}[Memoryless; Shakerinava \& Ravanbakhsh, 2022]
\label{axiom:memoryless}
For all $t\in \Tcal$ and $A, B \in\Delta(\Hcal)$
\begin{align*}
    A \prefge B \iff t\cdot A \prefge t \cdot B.
\end{align*}
\end{axiom}
\begin{restatable}{theorem}{memorylesstheorem}\label{thm:memoryless}
A binary preference relation on $\Delta(\Hcal)$ satisfies Axioms 1-5 where $\gamma(t)$ is relaxed to be in $\mathbb{R}_{\geq0}$ if and only if Axioms 1-4 and the Memoryless axiom are satisfied.
\end{restatable}

\begin{axiom}[Additivity; Shakerinava \& Ravanbakhsh, 2022]
\label{axiom:additivity}
For all $h_1, h_2 \in \Hcal$, $A, B, C, D \in \Delta(\Hcal)$ and $p \in [0, 1]$, 
\begin{align*}
    & p(h_1 \cdot A) + (1-p) C \prefge p(h_1 \cdot B) + (1-p) D \\
\Leftrightarrow & p(h_2 \cdot A) + (1-p) C \prefge p(h_2 \cdot B) + (1-p) D.
\end{align*}
\end{axiom}
\begin{restatable}{theorem}{additivitytheorem}
A binary preference relation on $\Delta(\Hcal)$ satisfies Axioms 1-5 with $\gamma(t) = 1$ if and only if Axioms 1-4 are satisfied as well as Additivity.
\end{restatable}

Separately, \citet{sunehag2011axioms,sunehag2015rationality} study what constitutes a \textit{rational reinforcement learning agent}. More concretely, in both works, Sunehag and Hutter suppose the environment and reward function are defined, and set out to characterize what kinds of agents might we consider rational.
In their first work \citep{sunehag2011axioms}, they provide a series of properties that characterize what it means for an RL agent to be rational. These properties bear some similarity to the vNM axioms, but are agent-side properties rather than implicit requirements on the structure of goals or rewards themselves.
In follow up work \citet{sunehag2015rationality} focus specifically on the rationality of optimistic agents. That is, given a reward function and environment, they contrast optimistic agents with those that are strictly expected utility maximizers. They give a full characterization of rational agency that justifies the use of optimism for exploration, showing that expected utility maximization on its own can be strictly worse than optimistic behavior.

\subsection{The Limited Expressivity of Markov Reward}

\begin{figure*}[!t]
    \centering
    \subfigure[\label{subfig:steady_state}Steady State Case]{\includegraphics[width=0.4\textwidth]{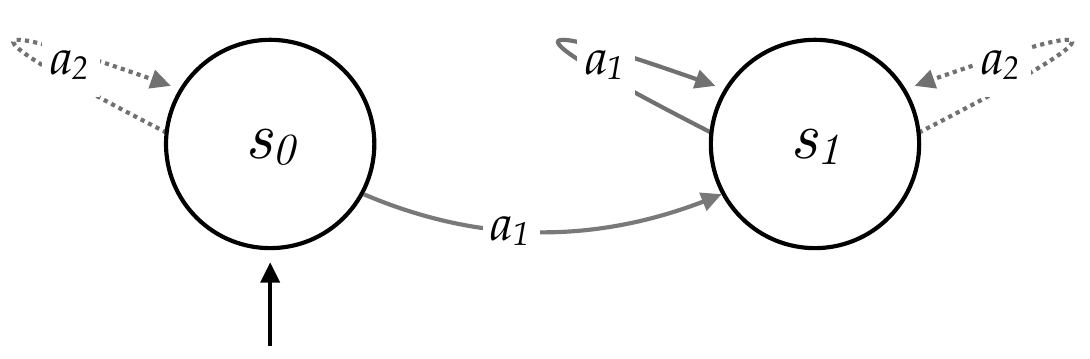}} \hspace{5mm}
    \subfigure[Entailment Case\label{subfig:entailment}]{\includegraphics[width=0.4\textwidth]{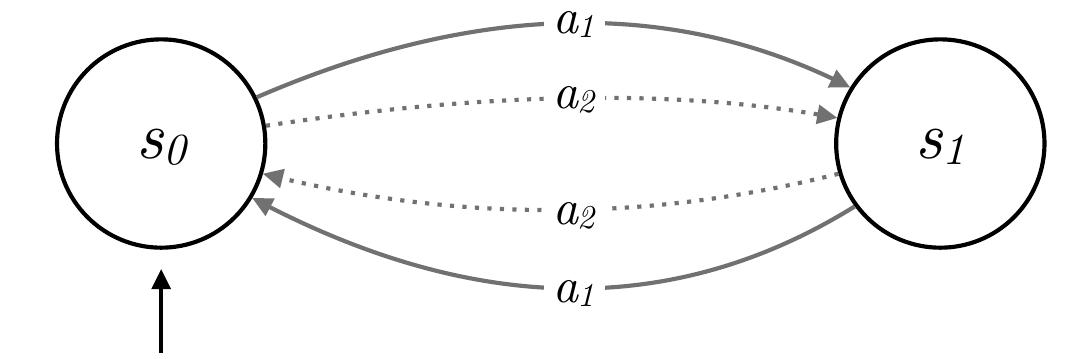}}
    
    \caption{The two counterexamples from \citet{abel2021reward}. In the steady state case, the set of acceptable policies contains only the policy that executes $a_2$ in the left, and $a_1$ in the right. In the entailment case, the two acceptable policies are those that choose a \textit{different} action across the two states.}
    \label{fig:example_inexpr_soap}
\end{figure*}

\citet{abel2021reward} study the expressivity of reward in Markovian environments. In particular, they suppose the environment is characterized by a finite state space and transition function, and assume that preferences over objects defined with respect to the environment are given. These preferences come in three forms: (1) A set of acceptable policies, (2) A partial ordering on policies, or (3) A partial ordering on fixed-length state-action trajectories. Each of these preference types are defined with respect to the environment's \textit{state space}, $\S$, that is known to be sufficient to support a Markovian transition function (and thus, $\env : \S \times \A \rightarrow \Delta(\S)$). For example, in the case of (1), the policy space is defined as all deterministic mappings of the form $\pi : \S \rightarrow \A$. Then, a choice of the first preference type is just a selection of acceptable policies. Under these three types, Abel et al. show that there are restrictions on what kinds of preferences can be codified in terms of a reward function that is Markov with respect to the same state space. Specifically, they point out two styles of counterexample: (1) the \textit{steady state type}, in which preferences have bearing on impossible outcomes, and (2) the \textit{entailment type}, in which the desirability of an action choice depends on behavior elsewhere in the environment.

We next show how the two styles of counterexample from Abel et al. play out in the context of our results. We find that the steady state type violates one of our assumptions, and the entailment type violates an axiom. 

\textbf{Steady State.  } First, recall the steady state counterexample, pictured in \autoref{subfig:steady_state}. 
The given preference over policies asserts that the only acceptable policy chooses $a_2$ in the left state, but $a_1$ in the right. This is a counterexample in the sense that there is no Markov reward function that ensures the policy $\pi_{22} : \{s_0 \mapsto a_2 \mid s_1 \mapsto a_2\}$, has higher value than $\pi_{21} : \{s_0 \mapsto a_2 \mid s_1 \mapsto a_1\}$, since both policies never reach state $s_1$.

Under our formalism, these policies would induce equivalent outcome distributions and thus be interchangeably preferred. 
We begin with preferences over outcomes of the kind suggested by \autoref{assumption:subjective-goals}.
That is, in the two-state environment described, our preferences are over sequences of state-action pairs. Recall that under \autoref{assumption:policy-prefs}, we understand a preference over policies $\pi_1 \prefge_g \pi_2$ to mean that there exists a time after which we always prefer the distribution over outcomes produced by $\pi_1$ to $\pi_2$.
Thus, when we consider the steady state counter example, we can see that \autoref{assumption:policy-prefs} is violated: We prefer $\pi_1$ to $\pi_2$ even though they induce \textit{identical} distributions $D_n^{\pi_1}$ and $D_n^{\pi_2}$. In this sense, the counterexample is an instance of a broader principle we have codified in our formalism.

\textbf{Entailment.  }
The second kind of counterexample deals with cases in which the value of one policy necessarily entails something about the value of another policy. For instance, in the example pictured in \autoref{subfig:entailment}, the two acceptable policies are those that take a different action in each of the two states. So, $\pi_{12} : \{s_0 \mapsto a_1 \mid s_1 \mapsto a_2\}$ and $\pi_{21} : \{s_0 \mapsto a_2 \mid s_1 \mapsto a_1\}$ are both desirable because they each take opposing actions across the two states. In follow up work, \citet{abelexpressing} demonstrate that a simple state-construction procedure can resolve the counterexample. In our framing, where precisely do the preferences of this entailment type go wrong? Such preferences directly violate \autoref{axiom:temporal_gamma_indifference}, the new axiom. To see why, consider the two (Dirac) distributions: $A = s_2,a_2$ and $B = s_2,a_1$. Let $t = s_1,a_1$. That is, the composite distributions formed are $t \cdot A = s_1,a_1,s_2,a_2$ and $t\cdot B = s_1,a_1, s_2,a_1$. However, there is no choice of $\gamma(t)$ for which the indifference expressed by the Axiom holds, as the preference requires that $A \prefge B$.

\section{Challenges to the Reward Hypothesis}

We next summarize common challenges to the reward hypothesis and consider whether our formalization of the hypothesis provides any further insight into these arguments.  

\subsection{Human Irrationality}
Extensive work by Kahneman \& Tversky and Johnson-Laird showed how human behavior deviates from the rational model proposed by von Neumann and Morgenstern \citep{kahneman1982judgment, kahneman1982psychology, tversky1983extensional,johnson1983mental}.
Based on these, one might conclude that our description of goals and purposes does not apply to humans, and is therefore incomplete. 
However, the expression of goals is distinct from the behaviors that emerge in their pursuit.
Our results prove the existence of a Markov reward signal under the presumption that all goals and purposes can be rationally expressed. 

\begin{figure*}[!t]
    \centering
    \includegraphics[]{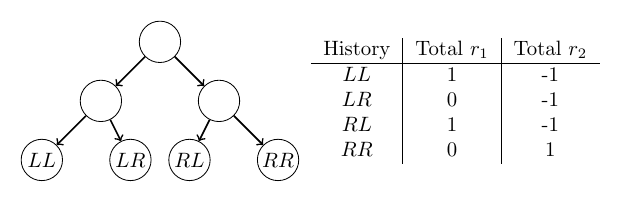}
    \caption{An environment used to illustrate how Constrained MDPs can violate independence and continuity.
    Two actions, $L$ and $R$, produce four histories shown in the table above.
    Also shown are the cumulative rewards of each history.
    }
    \label{fig:cmdp_counter_example}
\end{figure*}

\subsection{Multiple Objectives}

Another natural reaction to the reward hypothesis is to suspect that collapsing all of the nuance that might go into ``purpose'' down to a single scalar seems difficult, if not impossible.
Suppose we are interested in designing an autonomous taxi to take passengers between the airport and the university. We would like the taxi to balance between safety, timeliness, and energy use. But how precisely do we make these trade-offs, and can we really reduce the nuances of their trade-offs down to a single scalar? 

This challenge often yields a variant called multi-objective or multi-criteria decision making that has been studied extensively in the literature. \citet{hausner1953multidimensional} drops continuity (\autoref{axiom:continuity}) in order to generalize the vNM results to the multi-dimensional setting. \citet{gabor1998multi} propose multi-criteria RL in MDPs and establish initial conditions for importing classical results from scalar-valued rewards such as the existence of stationary policies and the Bellman equation. More recently, \citet{miura2022oteomdr} explicitly focus on the expressivity of multi-dimensional reward, enriching the result's of \citet{abel2021reward}. In particular, Miura shows that multi-dimensional Markov reward functions are strictly more expressive than scalar Markov reward functions in MDPs. However, we note that this expressivity comes with the cost of violating at least one of the axioms---we describe this in more detail shortly. \citet{pitisrational2022} make a similar argument, and prove that some multi-objective problems cannot be collapsed to a scalar objective.

\textbf{Constrained MDPs.  }
Constrained MDPs have been viewed as a challenge to the reward hypothesis \citep{szepesvariwebRLhypothesis}.
In a Constrained MDP, the goal is to maximize the expected sum of a base scalar reward, subject to additional constraints on the expected sums of other independent reward functions.
\citet{szepesvariwebRLhypothesis} showed that it is generally not possible to solve these objectives as typical MDPs, with a process of ``scalarization''.
One can show that Constrained MDPs do not always respect our notion of ``goals and purposes''. 

Consider the example pictured in Figure \ref{fig:cmdp_counter_example}.
The environment contains two actions, $L$ and $R$, whose combinations produce four different histories, denoted $LL,LR,RL,RR$.
The cumulative payoffs under the base reward $r_1$ and secondary reward $r_2$ are also shown.
Suppose the constrained objective involves maximizing expected utility under $r_1$ while demanding expected utility under $r_2$ be non-negative.
Under this objective, distributions of feasible histories are preferred to those which are infeasible.
Consider two such distributions $A= \frac{1}{2} RL + \frac{1}{2} RR$ and $B=LL$.
We have $A \prefge B$, since the first distribution is feasible and the second is not. 
Independence (\autoref{axiom:independence}) asserts that the preference remains unchanged when $A$ and $B$ are equally-mixed with any other distribution.
However, in this example, the preference reverses when mixed with $C=RR$.
That is $\frac{1}{2}A + \frac{1}{2}C \prefle \frac{1}{2} B + \frac{1}{2} C$. 
In this case, both distributions are feasible, but the first achieves less expected utility under the base reward $r_1$.

A similar example can be used for continuity (\autoref{axiom:continuity}). 
This time let $A = \frac{1}{2}RL + \frac{1}{2}RR$, $B=\frac{1}{2}LR + \frac{1}{2}RR$, and $C=RR$.
We have $A\prefge B \prefge C$, because $A$ is feasible with the highest expected base utility, $B$ is feasible but achieves less utility than $A$, and $C$ is infeasible.
For this selection, there is no break even point $p\in(0,1)$ that would make an agent indifferent between $A$ and the mixture $pB+(1-p)C$;
the latter is always infeasible.

\textbf{Risk.  }
Risk-sensitive goals provide another challenge to the reward hypothesis.
Risk-sensitive objectives applied to the returns generated by a policy can, in general, cause the optimal policy to be non-Markovian \citep{bdr2023}. This can be readily seen in the case of maximizing a variance-penalized mean return, where the agent first experiences some uncontrolled randomness that produces two different rewards (e.g. $0$ or $1$) and then faces the choice of two actions with different rewards (again, $0$ or $1$) allowing it to choose between maximizing value or reducing variance. For sufficiently large variance penalties the optimal policy will be to choose the action leading to the opposite reward of what it had previously experienced. However, this is not possible for Markov policies. 

Observe that when all optimal policies are non-Markov, such as in this example, there does not exist any Markov reward function with the same optimal policy under a risk-neutral objective. This is because the optimal policy will necessarily be Markovian. As a result, we can conclude that at least one of our assumptions or axioms must have been violated. In the above example, the issue is a violation of Axiom~\ref{axiom:temporal_gamma_indifference} and can be overcome by augmenting the state using the objective goals formulation.

\section{Conclusion}

We have here provided a conclusive account summarizing the implicit conditions required for the reward hypothesis. We separate such conditions into \textit{assumptions}, that center around interpretations of the hypothesis, and \textit{axioms}, that express specific formal properties about rational preference relations of all that one could mean by goals and purposes. Our main result (\autoref{thm:markov-reward}) states that, under Assumptions 1-4, the reward hypothesis for Markov reward functions holds if and only if Axioms 1-5 are satisfied. This result completely specifies the requirements on preferences under which the hypothesis holds. We further explore consequences of our new framing and results, including a variant from the viewpoint of the designer, an efficient constructive algorithm that translates rational preferences into a reward function, and discussion of how this axiomatic perspective can sharpen our understanding of alternative RL objectives such as constrained MDPs and risk.

\section{Acknowledgements}
We thank our colleagues, whose interactions and feedback shaped this paper. In particular, we thank Andr\'{e} Barreto and Doina Precup for thoughtful comments on an early draft, Bernardo {\'A}vila Pires for helpful conversations, and Niko Yasui for a thorough review of the text and for improvements to the constructive algorithm. We also thank W.~Bradley Knox, Serena Booth, Sarath Sreedharan, Stephane Hatgis-Kessell, Steven Shi, and Calarina Muslimani for feedback incorporated into this revision. Any remaining errors are those of the authors. This work was supported in part by NSERC, CIFAR, and Amii.

\bibliography{main}

\begin{thebibliography}{50}
\providecommand{\natexlab}[1]{#1}
\providecommand{\url}[1]{\texttt{#1}}
\expandafter\ifx\csname urlstyle\endcsname\relax
  \providecommand{\doi}[1]{doi: #1}\else
  \providecommand{\doi}{doi: \begingroup \urlstyle{rm}\Url}\fi

\bibitem[Abel et~al.(2021)Abel, Dabney, Harutyunyan, Ho, Littman, Precup, and
  Singh]{abel2021reward}
Abel, D., Dabney, W., Harutyunyan, A., Ho, M.~K., Littman, M.~L., Precup, D.,
  and Singh, S.
\newblock On the expressivity of {M}arkov reward.
\newblock In \emph{Advances in Neural Information Processing Systems}, 2021.

\bibitem[Abel et~al.(2022)Abel, Barreto, Bowling, Dabney, Hansen, Harutyunyan,
  Ho, Kumar, Littman, Precup, and Satinder]{abelexpressing}
Abel, D., Barreto, A., Bowling, M., Dabney, W., Hansen, S., Harutyunyan, A.,
  Ho, M.~K., Kumar, R., Littman, M.~L., Precup, D., and Satinder, S.
\newblock Expressing non-{M}arkov reward to a {M}arkov agent.
\newblock \emph{Multidisciplinary Conference on Reinforcement Learning and
  Decision Making}, 2022.

\bibitem[Aumann(1962)]{aumann1962utility}
Aumann, R.~J.
\newblock Utility theory without the completeness axiom.
\newblock \emph{Econometrica: Journal of the Econometric Society}, pp.\
  445--462, 1962.

\bibitem[Bellemare et~al.(2023)Bellemare, Dabney, and Rowland]{bdr2023}
Bellemare, M.~G., Dabney, W., and Rowland, M.
\newblock \emph{Distributional Reinforcement Learning}.
\newblock MIT Press, 2023.
\newblock \url{http://www.distributional-rl.org}.

\bibitem[Bernoulli(1738)]{bernoulli1738exposition}
Bernoulli, D.
\newblock Specimen theoriae novae de mensura sortis (trans. in 1954 as
  exposition of a new theory on the measurement of risk).
\newblock \emph{Econometrica}, 22\penalty0 (1):\penalty0 23--36, 1738.

\bibitem[Cassandra et~al.(1994)Cassandra, Kaelbling, and
  Littman]{cassandra1994acting}
Cassandra, A.~R., Kaelbling, L.~P., and Littman, M.~L.
\newblock Acting optimally in partially observable stochastic domains.
\newblock In \emph{Proceedings of the AAAI Conference on Artificiall
  Intelligence}, 1994.

\bibitem[Chang(2015)]{chang2015value}
Chang, R.
\newblock Value incomparability and incommensurability.
\newblock \emph{The {O}xford handbook of value theory}, pp.\  205--224, 2015.

\bibitem[Dong et~al.(2021)Dong, Van~Roy, and Zhou]{dong2021simple}
Dong, S., Van~Roy, B., and Zhou, Z.
\newblock Simple agent, complex environment: Efficient reinforcement learning
  with agent states.
\newblock \emph{arXiv preprint arXiv:2102.05261}, 2021.

\bibitem[Fedus et~al.(2019)Fedus, Gelada, Bengio, Bellemare, and
  Larochelle]{fedus2019hyperbolic}
Fedus, W., Gelada, C., Bengio, Y., Bellemare, M.~G., and Larochelle, H.
\newblock Hyperbolic discounting and learning over multiple horizons.
\newblock \emph{arXiv preprint arXiv:1902.06865}, 2019.

\bibitem[G{\'a}bor et~al.(1998)G{\'a}bor, Kalm{\'a}r, and
  Szepesv{\'a}ri]{gabor1998multi}
G{\'a}bor, Z., Kalm{\'a}r, Z., and Szepesv{\'a}ri, C.
\newblock Multi-criteria reinforcement learning.
\newblock In \emph{Proceedings of the International Conference on Machine
  Learning}, volume~98, 1998.

\bibitem[Hausner(1953)]{hausner1953multidimensional}
Hausner, M.
\newblock Multidimensional utilities (rev).
\newblock Technical report, RAND CORP SANTA MONICA CA, 1953.

\bibitem[Johnson-Laird(1983)]{johnson1983mental}
Johnson-Laird, P.~N.
\newblock \emph{Mental models: Towards a cognitive science of language,
  inference, and consciousness}.
\newblock Harvard University Press, 1983.

\bibitem[Kahneman \& Tversky(1982)Kahneman and Tversky]{kahneman1982psychology}
Kahneman, D. and Tversky, A.
\newblock The psychology of preferences.
\newblock \emph{Scientific American}, 246\penalty0 (1):\penalty0 160--173,
  1982.

\bibitem[Kahneman et~al.(1982)Kahneman, Slovic, Slovic, and
  Tversky]{kahneman1982judgment}
Kahneman, D., Slovic, S.~P., Slovic, P., and Tversky, A.
\newblock \emph{Judgment under uncertainty: Heuristics and biases}.
\newblock Cambridge university press, 1982.

\bibitem[Keeney et~al.(1993)Keeney, Raiffa, and Meyer]{keeney1993decisions}
Keeney, R.~L., Raiffa, H., and Meyer, R.~F.
\newblock \emph{Decisions with multiple objectives: preferences and value
  trade-offs}.
\newblock Cambridge university press, 1993.

\bibitem[Koopmans(1960)]{koopmans1960stationary}
Koopmans, T.~C.
\newblock Stationary ordinal utility and impatience.
\newblock \emph{Econometrica: Journal of the Econometric Society}, pp.\
  287--309, 1960.

\bibitem[Koopmans et~al.(1964)Koopmans, Diamond, and
  Williamson]{koopmans1964stationary}
Koopmans, T.~C., Diamond, P.~A., and Williamson, R.~E.
\newblock Stationary utility and time perspective.
\newblock \emph{Econometrica: Journal of the Econometric Society}, pp.\
  82--100, 1964.

\bibitem[Kreps \& Porteus(1978)Kreps and Porteus]{kreps1978temporal}
Kreps, D.~M. and Porteus, E.~L.
\newblock Temporal resolution of uncertainty and dynamic choice theory.
\newblock \emph{Econometrica: journal of the Econometric Society}, pp.\
  185--200, 1978.

\bibitem[Lattimore(2014)]{lattimore2014theory}
Lattimore, T.
\newblock \emph{Theory of General Reinforcement Learning}.
\newblock PhD thesis, The Australian National University, 2014.

\bibitem[Lattimore et~al.(2013)Lattimore, Hutter, and
  Sunehag]{lattimore2013sample}
Lattimore, T., Hutter, M., and Sunehag, P.
\newblock The sample-complexity of general reinforcement learning.
\newblock In \emph{Proceedings of the International Conference on Machine
  Learning}, 2013.

\bibitem[Leike(2016)]{leike2016nonparametric}
Leike, J.
\newblock \emph{Nonparametric general reinforcement learning}.
\newblock PhD thesis, The Australian National University, 2016.

\bibitem[Leike et~al.(2018)Leike, Krueger, Everitt, Martic, Maini, and
  Legg]{leike2018scalable}
Leike, J., Krueger, D., Everitt, T., Martic, M., Maini, V., and Legg, S.
\newblock Scalable agent alignment via reward modeling: a research direction.
\newblock \emph{arXiv preprint arXiv:1811.07871}, 2018.

\bibitem[Lewis(1985)]{lewis1985effectively}
Lewis, A.~A.
\newblock On effectively computable realizations of choice functions: Dedicated
  to professors kenneth j. arrow and anil nerode.
\newblock \emph{Mathematical Social Sciences}, 10\penalty0 (1):\penalty0
  43--80, 1985.

\bibitem[Littman(2017)]{littmanwebRLhypothesis}
Littman, M.
\newblock The reward hypothesis.
\newblock \url{https://tinyurl.com/4z52r3fe}, 2017.

\bibitem[Lu et~al.(2021)Lu, Van~Roy, Dwaracherla, Ibrahimi, Osband, and
  Wen]{lu2021reinforcement}
Lu, X., Van~Roy, B., Dwaracherla, V., Ibrahimi, M., Osband, I., and Wen, Z.
\newblock Reinforcement learning, bit by bit.
\newblock \emph{arXiv preprint arXiv:2103.04047}, 2021.

\bibitem[Machina(1990)]{machina1990expected}
Machina, M.~J.
\newblock Expected utility hypothesis.
\newblock In \emph{Utility and probability}, pp.\  79--95. Springer, 1990.

\bibitem[Mahadevan(1996)]{mahadevan1996average}
Mahadevan, S.
\newblock Average reward reinforcement learning: Foundations, algorithms, and
  empirical results.
\newblock \emph{Machine learning}, 22\penalty0 (1):\penalty0 159--195, 1996.

\bibitem[Majeed(2021)]{majeed2021abstractions}
Majeed, S.~J.
\newblock \emph{Abstractions of General Reinforcement Learning}.
\newblock PhD thesis, The Australian National University, 2021.

\bibitem[Martin(2011)]{martin2011st}
Martin, R.
\newblock The st. petersburg paradox.
\newblock \emph{Stanford Encyclopedia of Philosophy}, 2011.

\bibitem[McCarthy(1998)]{mccarthy1998artificial}
McCarthy, J.
\newblock What is artificial intelligence.
\newblock \emph{URL: http://www-formal. stanford. edu/jmc/whatisai. html},
  1998.

\bibitem[Miura(2022)]{miura2022oteomdr}
Miura, S.
\newblock On the expressivity of multidimensional {M}arkov reward.
\newblock In \emph{In RLDM Workshop on Reinforcement Learning as a Model of
  Agency}, 2022.

\bibitem[Ng et~al.(2000)Ng, Russell, et~al.]{ng2000algorithms}
Ng, A.~Y., Russell, S., et~al.
\newblock Algorithms for inverse reinforcement learning.
\newblock In \emph{Proceedings of the International Conference on Machine
  Learning}, pp.\  663--670, 2000.

\bibitem[Pitis(2019)]{pitis2019rethinking}
Pitis, S.
\newblock Rethinking the discount factor in reinforcement learning: A decision
  theoretic approach.
\newblock In \emph{Proceedings of the AAAI Conference on Artificial
  Intelligence}, 2019.

\bibitem[Pitis et~al.(2022)Pitis, Bailey, and Ba]{pitisrational2022}
Pitis, S., Bailey, D., and Ba, J.
\newblock Rational multi-objective agents must admit non-markov reward
  representations.
\newblock \emph{NeurIPS Workshop on Machine Learning Safety}, 2022.

\bibitem[Ramsey(1926)]{ramsey2016truth}
Ramsey, F.~P.
\newblock Truth and probability.
\newblock In \emph{Readings in formal epistemology}, pp.\  21--45. Springer,
  1926.

\bibitem[Richter \& Wong(1999)Richter and Wong]{richter1999computable}
Richter, M.~K. and Wong, K.-C.
\newblock Computable preference and utility.
\newblock \emph{Journal of Mathematical Economics}, 32\penalty0 (3):\penalty0
  339--354, 1999.

\bibitem[Rustem \& Velupillai(1990)Rustem and
  Velupillai]{rustem1990rationality}
Rustem, B. and Velupillai, K.
\newblock Rationality, computability, and complexity.
\newblock \emph{Journal of Economic Dynamics and Control}, 14\penalty0
  (2):\penalty0 419--432, 1990.

\bibitem[Shakerinava \& Ravanbakhsh(2022)Shakerinava and
  Ravanbakhsh]{shakerinava2022utility}
Shakerinava, M. and Ravanbakhsh, S.
\newblock Utility theory for sequential decision making.
\newblock In \emph{Proceedings of the International Conference on Machine
  Learning}, 2022.

\bibitem[Silver et~al.(2021)Silver, Singh, Precup, and
  Sutton]{silver2021reward}
Silver, D., Singh, S., Precup, D., and Sutton, R.~S.
\newblock Reward is enough.
\newblock \emph{Artificial Intelligence}, 299:\penalty0 103535, 2021.

\bibitem[Singh et~al.(2009)Singh, Lewis, and Barto]{singh2009rewards}
Singh, S., Lewis, R.~L., and Barto, A.~G.
\newblock Where do rewards come from?
\newblock In \emph{Proceedings of the Annual Conference of the Cognitive
  Science Society}, 2009.

\bibitem[Sobel(1975)]{sobel1975ordinal}
Sobel, M.~J.
\newblock Ordinal dynamic programming.
\newblock \emph{Management science}, 21\penalty0 (9):\penalty0 967--975, 1975.

\bibitem[Sunehag \& Hutter(2011)Sunehag and Hutter]{sunehag2011axioms}
Sunehag, P. and Hutter, M.
\newblock Axioms for rational reinforcement learning.
\newblock In \emph{International Conference on Algorithmic Learning Theory},
  pp.\  338--352. Springer, 2011.

\bibitem[Sunehag \& Hutter(2015)Sunehag and Hutter]{sunehag2015rationality}
Sunehag, P. and Hutter, M.
\newblock Rationality, optimism and guarantees in general reinforcement
  learning.
\newblock \emph{The Journal of Machine Learning Research}, 16\penalty0
  (1):\penalty0 1345--1390, 2015.

\bibitem[Sutton(2004)]{suttonwebRLhypothesis}
Sutton, R.~S.
\newblock The reward hypothesis.
\newblock
  \url{http://incompleteideas.net/rlai.cs.ualberta.ca/RLAI/rewardhypothesis.html},
  2004.

\bibitem[Sutton \& Barto(2018)Sutton and Barto]{sutton2018reinforcement}
Sutton, R.~S. and Barto, A.~G.
\newblock \emph{Reinforcement learning: An introduction}.
\newblock MIT press, 2018.

\bibitem[Szepesv{\' a}ri(2020)]{szepesvariwebRLhypothesis}
Szepesv{\' a}ri, C.
\newblock Constrained {MDP}s and the reward hypothesis.
\newblock
  \url{https://readingsml.blogspot.com/2020/03/constrained-mdps-and-reward-hypothesis.html},
  2020.

\bibitem[Tversky \& Kahneman(1983)Tversky and Kahneman]{tversky1983extensional}
Tversky, A. and Kahneman, D.
\newblock Extensional versus intuitive reasoning: The conjunction fallacy in
  probability judgment.
\newblock \emph{Psychological review}, 90\penalty0 (4):\penalty0 293, 1983.

\bibitem[von Neumann \& Morgenstern(1953)von Neumann and
  Morgenstern]{vonneumann1953theory}
von Neumann, J. and Morgenstern, O.
\newblock \emph{Theory of Games and Economic Behavior}.
\newblock Princeton University Press, third edition, 1953.

\bibitem[White(2017)]{white17}
White, M.
\newblock Unifying task specification in reinforcement learning.
\newblock In \emph{Proceedings of the Thirty-fourth International Conference on
  Machine Learning}, 2017.

\bibitem[Wirth et~al.(2017)Wirth, Akrour, Neumann, F{\"u}rnkranz,
  et~al.]{wirth2017survey}
Wirth, C., Akrour, R., Neumann, G., F{\"u}rnkranz, J., et~al.
\newblock A survey of preference-based reinforcement learning methods.
\newblock \emph{Journal of Machine Learning Research}, 18\penalty0
  (136):\penalty0 1--46, 2017.

\end{thebibliography}
\bibliographystyle{icml2023}

\appendix
\onecolumn

\section{Proof of Main Theorem}

\markovrewardtheorem*
\begin{proof}
We first show the axioms imply the existence of a Markov reward.  By Theorem~\ref{thm:vnm} and Axioms 1-4, we can select $u : \Delta(\hists) \rightarrow \mathbb{R}$, such that $u(\emptytraj) = 0$, leaving $u$ unique up to a positive scale factor.  Define $r(t) \defeq u(t)$.  From Axiom 5 and choosing $h_1 = h$ and $h_2 = \emptytraj$, 
\begin{align*}
    \frac{1}{\gamma(t) + 1}(t \cdot h) + \frac{\gamma(t)}{\gamma(t) + 1} \emptytraj \sim
    \frac{1}{\gamma(t) + 1}(t \cdot \varepsilon) + \frac{\gamma(t)}{\gamma(t) + 1} h 
\end{align*}
Applying Theorem~\ref{thm:vnm} (first consequence 1 then consequence 2) we get,
\begin{align*}
    \frac{1}{\gamma(t) + 1}u(t \cdot h) + \frac{\gamma(t)}{\gamma(t) + 1} u(\emptytraj) =  
    \frac{1}{\gamma(t) + 1}u(t) + \frac{\gamma(t)}{\gamma(t) + 1} u(h) 
\end{align*}
Multiplying by $\gamma(t)+1$, we get, 
\begin{align*}
    u(t \cdot h) + \gamma(t) u(\emptytraj) &= u(t) + \gamma(t) u(h),\\
    u(t \cdot h) &= r(t) + \gamma(t) u(h)
\end{align*}

We now show that any Markov reward satisfies the axioms. Suppose $u$, $r$, and $\gamma$ satisfy conditions 1--4. Theorem~\ref{thm:vnm} guarantees that Axioms 1-4 are satisfied. For Axiom 5, conditions 2 and 4 satisfy the hypotheses of \autoref{lemma:dist_sub}, so for every $t \in \Tcal$ and $A \in \Delta(\hists)$,
\[
    u(t \cdot A) = r(t) + \gamma(t) u(A),
\]
so,
\[
    u(t \cdot A) -  \gamma(t) u(A) = r(t).
\]
Hence for all $A, B \in \Delta(\hists)$
\[
    u(t \cdot A) -  \gamma(t) u(A) = r(t) = u(t \cdot B) -  \gamma(t) u(B).
\]
Rearranging we get,
\[
    u(t \cdot A) +  \gamma(t) u(B) = u(t \cdot B) + \gamma(t) u(A).
\]
Dividing all terms by $\gamma(t) + 1$,
\begin{align*}
    \frac{1}{\gamma(t)+1} u(t \cdot A) +  \frac{\gamma(t)}{\gamma(t)+1} u(B) =
    \frac{1}{\gamma(t)+1} u(t \cdot B) + \frac{\gamma(t)}{\gamma(t)+1} u(A).
\end{align*}
Each side is a mixture of distributions over histories, which is itself a distribution over histories, so by condition 4 each side is the utility of the corresponding mixture. Applying condition 3, we find that Axiom 5 is satisfied:
\begin{align*}
\frac{1}{\gamma(t) + 1}(t \cdot A) + \frac{\gamma(t)}{\gamma(t) + 1}B \sim
\frac{1}{\gamma(t) + 1}(t \cdot B) + \frac{\gamma(t)}{\gamma(t) + 1}A.
\end{align*}

\end{proof}

\begin{lemma}[Distributional Subadditivity]\label{lemma:dist_sub}
Let $u : \Delta(\hists) \rightarrow \mathbb{R}$ be a utility function with the following properties
\begin{enumerate}
    \item $\forall \ t\in \Tcal \text{ and } h\in \hists, u(t\cdot h) = r(t) + \gamma(t)u(h)$ for some $r\colon \Tcal \rightarrow \mathbb{R}$ and $\gamma\colon \Tcal \rightarrow [0,1]$, where $u(h)=u(\delta_h)$,
    \item $\forall P \in\Delta(\hists), \sum_{h\in\Hcal} P_h u(h)$ exists and $u(P) = \sum_{h\in\Hcal} P_h u(h)$,
\end{enumerate}
then for all $t\in \Tcal$ and $P\in\Delta(\hists)$,  $u(t \cdot P) = r(t) + \gamma(t) u(P).$
\end{lemma}

\begin{proof}
For any $t \in \Tcal$ and $P \in \Delta(\hists)$, the distribution $t \cdot P$ is the pushforward of $P$ under the injective map $h \mapsto t \cdot h$. Hence $(t \cdot P)_{t \cdot h} = P_h$ for every $h \in \Hcal$, and $(t \cdot P)_{h'} = 0$ for every $h'$ that does not begin with $t$.

By property~2 applied to $t \cdot P$,
\[
  u(t \cdot P) = \sum_{h' \in \Hcal} (t \cdot P)_{h'} u(h').
\]
Every term with $h'$ not beginning with $t$ is zero, so reindexing by $h' = t \cdot h$
gives
\[
  u(t \cdot P) = \sum_{h \in \Hcal} (t \cdot P)_{t \cdot h} u(t \cdot h)
  = \sum_{h \in \Hcal} P_h\, u(t \cdot h).
\]

By property~1, $u(t \cdot h) = r(t) + \gamma(t) u(h)$ for each $h$. Substituting,
\[
  u(t \cdot P) = \sum_{h \in \Hcal} \Bigl(P_h r(t) + P_h \gamma(t) u(h) \Bigr).
\]

Both $\sum_{h\in\Hcal}P_h r(t)$ and $\sum_{h\in\Hcal}P_h \gamma(t) u(h)$ exist. The first is $\sum_{h\in\Hcal} r(t) P_h = r(t)$ since $\sum_{h\in\Hcal} P_h = 1$. The second, $\sum_{h\in\Hcal} P_h u(h)$, follows from property~2. So the sum splits as
\[
  u(t \cdot P) = r(t) + \gamma(t) \sum_{h \in \Hcal} P_h\, u(h).
\]
Finally, $\sum_{h \in \Hcal} P_h u(h) = u(P)$ by property~2. Therefore
\[
  u(t \cdot P) = r(t) + \gamma(t)\, u(P). \qedhere
\]
\end{proof}
\section{Proofs of Relationship to \cite{shakerinava2022utility}}

For the following proofs let $u_t(x) \defeq u(t \cdot x)$ for all $t \in \Ocal\times \Acal$.

\memorylesstheorem*

\begin{proof}
	Starting from the Temporal $\gamma$-Indifference Axiom, we use the vNM theorem under a specific choice of history to reduce the utility function to a positive affine form. 

	$\implies$ Take some $t$ from $\Ocal\times \Acal$, and let $\gamma(t)>0$.
	Temporal $\gamma$-Indifference states that for any distributions $A$ and $B$ from $\Delta(\Hcal)$
	\begin{align*}
				\frac{1}{1+\gamma(t)}\ (t\cdot A) + \frac{\gamma(t)}{1+\gamma(t)}\ B 
				\sim
		\frac{1}{1+\gamma(t)}\ (t\cdot B) + \frac{\gamma(t)}{1+\gamma(t)}\ A .
	\end{align*}
	The vNM utility theorem guarantees this indifference has a utility representation:
	\begin{align*}
		\frac{1}{1+\gamma(t)}u(t \cdot A) + \frac{\gamma(t)}{1+\gamma(t)}u(B) 
		&=
		\frac{1}{1+\gamma(t)}u(t \cdot B) + \frac{\gamma(t)}{1+\gamma(t)}u(A),\\
		u(t \cdot A) + \gamma(t)u(B) &= u(t \cdot B) + \gamma(t)u(A).
	\end{align*}
	Let $B=\epsilon$, and define $u(\epsilon)\triangleq 0$, $r(t)\triangleq u(t\cdot \epsilon)$ so we obtain  
	\begin{align*}
		u(t \cdot A) &= r(t) + \gamma(t)u(A).
	\end{align*}

    Using Lemma \ref{lem:uni_affine} we get that $u_t$ is strategically equivalent to $u$, where $a=r(t)$ and $b=\gamma(t)$.  
	Hence, for all distributions $A,B\in\Delta(\Hcal)$ 
	\begin{align*}
	    u(A) \geq u(B) \iff u_t(A) \geq u_t(B),
	\end{align*}
	therefore,
	\begin{align*}
		A \prefge B \iff (t \cdot A) \prefge (t \cdot B).
	\end{align*}

	$\impliedby$ The Memoryless Axiom states that for all $t\in \Tcal$ and $A,B\in\Delta(\Hcal)$
	\begin{align*}
		A \prefge B \iff (t \cdot A) \prefge (t \cdot B)
	\end{align*}
	This means that $u_t$ is strategically equivalent to $u$. 
	By Lemma \ref{lem:uni_affine}, we know there exists constants, which we label $r(t)$ and $\gamma(t)$, such that,
    \begin{align*}
		u(t \cdot A) = u_t(A) &= r(t) +\gamma(t)u(A), 	&	
		u(t \cdot B) = u_t(B) &= r(t) +\gamma(t)u(B).
	\end{align*}
	Simple algebra shows that
	\begin{align*}
		u(t \cdot A) - \gamma(t)u(A) &= u(t \cdot B) - \gamma(t)u(B),\\
		u(t \cdot A) + \gamma(t)u(B) &= u(t \cdot B) + \gamma(t)u(A)\\
		\frac{1}{1+\gamma(t)}u(t \cdot A) + \frac{\gamma(t)}{1+\gamma(t)}u(B) 
		&= \frac{1}{1+\gamma(t)}u(t \cdot B) + \frac{\gamma(t)}{1+\gamma(t)}u(A),\\		
		\frac{1}{1+\gamma(t)}(t \cdot A) + \frac{\gamma(t)}{1+\gamma(t)}B 
		&\sim \frac{1}{1+\gamma(t)}(t \cdot B) + \frac{\gamma(t)}{1+\gamma(t)}A.		
	\end{align*}
	The last line follows from the vNM utility theorem.
\end{proof}

\begin{definition}[Strategic Equivalence]
	Two utility functions $u_1$ and $u_2$ are strategically equivalent if and only if they imply the same preference ranking for any two distributions.
\end{definition}

\begin{lemma}\label{lem:uni_affine}
    Two utility functions $u_1$ and $u_2$ are strategically equivalent if and only if there exists two constants $a$ and $b>0$ such that
	\begin{align*}
		u_1(x) = a + b u_2(x) \text{ for all $x$.}
	\end{align*}
\end{lemma}
\begin{proof} 
    The first direction follows a similar argument to Theorem 4.1 from \citet{keeney1993decisions}.

    $\implies$ Let $x_0 \in \argmin_x u_2(x)$ and $x^* \in \argmax_x u_2(x)$.  We first consider the degenerate case where $u_2(x_0) = u_2(x^*)$, so $u_2$ is a constant function.  By strategic equivalence $u_1$ must also be a constant function, trivially satisfying the implication.  Now consider the case where $u_2(x^*) > u_2(x_0)$.
    For any $x$, there exists $c\in [0, 1]$ such that $x \sim c x^* + (1-c) x_0$ under $u_2$, and by strategic equivalence $u_1$ as well.  Therefore,
    \begin{align*}
		u_i(x) = c u_i(x^*) + (1-c)u_i(x_0) \text{ for $i=1,2$.}
	\end{align*}
	Letting $i=2$ and solving for $c$ we get
	\begin{align*}
		c = \frac{u_2(x)-u_2(x_0)}{u_2(x^*)-u_2(x_0)}.
	\end{align*}
	Substituting this value of $c$ in when $i=1$ gives.
	\begin{align*}
	    u_1(x) &= \left(\frac{u_2(x)-u_2(x_0)}{u_2(x^*)-u_2(x_0)} \right)u_1(x^*) + \left(1 -  \frac{u_2(x)-u_2(x_0)}{u_2(x^*)-u_2(x_0)}\right)u_1(x_0),\\
	    &= \underbrace{\left(u_1(x_0)+ \frac{u_2(x_0)u_1(x_0)-u_2(x_0)u_1(x^*)}{u_2(x^*)-u_2(x_0)} \right)}_a + \underbrace{\left(\frac{u_1(x^*) - u_1(x_0)}{u_2(x^*)-u_2(x_0)} \right)}_{b}u_2(x).
	\end{align*}
    Notice that $b>0$, because $u_2(x^*) > u_2(x_0)$.

	$\impliedby$ Assume there exists constants $a$ and $b>0$ such that 
	\begin{align*}
		u_1(x) = a + b u_2(x) \text{ for all $x$.}
	\end{align*}
    Take any two $x$ and $x'$ such that $u_2(x) > u_2(x')$.
    Scaling by a positive constant $b$ and shifting the utility by $a$ leaves the relation unchanged.
    Therefore,
	\begin{align*}
	    u_2(x) &> u_2(x'),\\
		a + b u_2(x) &> a + b u_2(x'),\\
		u_1(x) &> u_1(x').
	\end{align*}
\end{proof}

\additivitytheorem*

\begin{proof}
	The first direction follows from Independence and Lemma \ref{lem:equiv}.
	The converse is reduced to the Memoryless Axiom, then the remainder of the argument follows from Theorem \ref{thm:memoryless}.
	
	$\implies$ Take any $t\in\Ocal\times \Acal$, and $A,B,C,D\in\Delta(\Hcal)$ for which
	\begin{align*}
		p(t \cdot A) + (1-p)C \prefge p(t \cdot B) + (1-p)D.
	\end{align*}
	By the Independence Axiom, the preference remains unchanged after mixing distributions with $(t'\cdot X)$ by an amount $q\in[0,1]$:
	\begin{align*}
		qp(t \cdot A) &+ q(1-p)C +(1-q)(t'\cdot X) \prefge qp(t \cdot B) + q(1-p)D + (1-q)(t'\cdot X).
	\end{align*}
	If we take $qp = (1-q)$, then we can form uniform compound distributions of $(t,\cdot)$ and $(t'\cdot X)$:
	\begin{align*}
		qp[ (t \cdot A) + (t'\cdot X) ] + q(1-p)C
		&\prefge qp[(t \cdot B)+ (t'\cdot X)] + q(1-p)D,\\
		2qp\left[ \frac{1}{2}(t \cdot A) + \frac{1}{2}(t'\cdot X) \right] + q(1-p)C 
		&\prefge 2qp\left[\frac{1}{2}(t \cdot B) + \frac{1}{2}(t'\cdot X)\right] + q(1-p)D,\\
		2qp\left[ \frac{1}{2}(t'\cdot A) + \frac{1}{2}(t\cdot X) \right] + q(1-p)C 
		&\prefge 2qp\left[\frac{1}{2}(t',B) + \frac{1}{2}(t\cdot X)\right] + q(1-p)D.
	\end{align*}
	The last line follows from Lemma \ref{lem:equiv}, which takes Temporal Indifference as its precondition.
	
	Finally, we invoke the Independence axiom to remove $(t\cdot X)$ mixture:
	\begin{align*}
		p (t'\cdot A) + (1-p)C &\prefge p(t',B)+ (1-p)D.
	\end{align*}

	$\impliedby$ For any $t,t'\in \Tcal$, and $A,B,C,D\in\Delta(\Hcal)$, the Additivity Axiom states
	\begin{align*}
		p&(t, A) + (1-p) C \prefge p(t \cdot B) +(1-p)D
		\iff  p(t', A) + (1-p) C \prefge p(t',B) +(1-p)D.
	\end{align*}
	This reduces to the Memoryless Axiom if we restrict $t$ to $\Ocal\times\Acal$, and take $t'=\epsilon$ and $C=D$:
	\begin{align*}
		p(t, A) + (1-p) D \prefge p(t \cdot B) +(1-p)D
		&\iff  p(\epsilon, A) + (1-p) D \prefge p(\epsilon,B) +(1-p)D,\\
 		p(t, A) + (1-p) D \prefge p(t \cdot B) +(1-p)D
 		&\iff  p A + (1-p) D \prefge p B +(1-p)D,\\
  		(t, A) \prefge (t \cdot B) &\iff  A  \prefge B.
	\end{align*}
	The last line follows from independence on $D$.
	
	Finally, Theorem \ref{thm:memoryless} established that the Memoryless Axiom holds if and only if the Temporal $\gamma$-Indifference Axiom holds. Therefore, the remainder of the proof follows from that.
\end{proof}

\begin{lemma}\label{lem:equiv}
	If Temporal $\gamma$-Indifference holds when $\gamma=1$, along with Axioms 1-5, then for any $t,t'\in \Ocal\times \Acal$, and $A,X\in\Delta(\Hcal)$, 
	\begin{align*}
				\frac{1}{2}(t \cdot A) + \frac{1}{2}(t'\cdot X) &\sim \frac{1}{2}(t'\cdot A) + \frac{1}{2}(t\cdot X).
	\end{align*}
	\begin{proof}
	Take some $t$ from $\Ocal\times \Acal$, and let $\gamma(t)=1$.
	Temporal $\gamma$-Indifference states that for any distributions $A$ and $X$
	\begin{align*}
		\frac{1}{2} (t \cdot A)+ \frac{1}{2}X &\sim \frac{1}{2} (t\cdot X)+ \frac{1}{2}A.
	\end{align*}
	This applies to any $t$, so it must apply to any other $t'$ from $\Ocal\times \Acal$ with $\gamma(t')=1$:
	\begin{align*}
		\frac{1}{2} (t \cdot A)+ \frac{1}{2}X \sim \frac{1}{2} (t\cdot X)+ \frac{1}{2}A 
		&\iff \frac{1}{2} (t'\cdot A)+ \frac{1}{2}X \sim \frac{1}{2} (t'\cdot X)+ \frac{1}{2}A.
	\end{align*}
	The vNM utility theorem guarantees these preferences have a utility representation, meaning:
	\begin{align*}
		\frac{1}{2} u(t \cdot A)+ \frac{1}{2}u(X) = \frac{1}{2} u(t\cdot X)+ \frac{1}{2}u(A) 
		&\iff \frac{1}{2} u(t'\cdot A)+ \frac{1}{2}u(X) = \frac{1}{2} u(t'\cdot X)+ \frac{1}{2}u(A).
	\end{align*}
	Regardless of whether a distribution involves $t$ or $t'$, the difference between utilities will be equal:
	\begin{align*}
		u(t \cdot A)+ u(X) - u(t\cdot X) - u(A)
		&= u(t'\cdot A)+ u(X) - u(t'\cdot X) - u(A),\nonumber\\
		u(t \cdot A) - u(t\cdot X) &= u(t'\cdot A) - u(t'\cdot X),\nonumber\\
		u(t \cdot A) + u(t'\cdot X) &= u(t'\cdot A) + u(t\cdot X),\nonumber\\
		u\left( \frac{1}{2}(t \cdot A) + \frac{1}{2}(t'\cdot X)\right) &= u\left(\frac{1}{2}(t'\cdot A) + \frac{1}{2}(t\cdot X)\right),\nonumber\\
		\frac{1}{2}(t \cdot A) + \frac{1}{2}(t'\cdot X) &\sim \frac{1}{2}(t'\cdot A) + \frac{1}{2}(t\cdot X).\label{eq:equiv}
	\end{align*}
	\end{proof}
\end{lemma}

\section{Preferences in the Average Reward Setting}

The notion of the cumulative sum eventually being larger allows us to capture settings where the series is convergent (i.e., $\lim_{n\rightarrow\infty} V^\pi_n$ exists) as well as cases where it is not, such as average reward; if one policy has higher average reward, then its finite sum must eventually be larger.  For this analysis we will assume without loss of generality that the transition-dependent discount $\gamma(t) = 1 \;\forall t\in \Tcal$.

This last point is made formal in the following proposition, where the average reward for policy $\pi$ is $\mu_\pi \triangleq \lim_{n\rightarrow\infty}\frac{1}{n}V^{\pi}_n$.
\begin{restatable}{proposition}{avgreward}\label{prop:average_reward}
	For any policies $\pi_A,\pi_B$ whose average rewards $\mu_A, \mu_B$ exist, if  $\mu_A > \mu_B$, then there exists an $N$ such that $V^A_n > V^B_n$ for all $n\geq N$.
\end{restatable}
\begin{proof}
	Assume $\mu_A > \mu_B$ exist. According to the definition of a limit, for any $\varepsilon_A,\varepsilon_B>0$, there exists some $N_A$ and $N_B$ such that 
	\begin{align*}
		\mu_A-\varepsilon_A < &\frac{1}{k}V_k^{A} &
		\mu_B+\varepsilon_B > &\frac{1}{m}V_m^{B},
	\end{align*}
	for all $k\geq N_A$ and $m \geq N_B$. 
	
	Choose $\varepsilon_A=\varepsilon_B=\frac{1}{2}(\mu_A - \mu_B) > 0$ and let $N=\max\{N_A,N_B\}$, so the above inequalities hold for all $n\geq N$.
	If you negate the second inequality and add them together, then for all $n > N$,
	\begin{align*}
	    (\mu_A - \varepsilon_A) - (\mu_B + \varepsilon_B) <&
	    \frac{1}{n}V^A_n - \frac{1}{n}V^B_n \\
	    (\mu_A - \mu_B) - (\varepsilon_A + \varepsilon_B) <&
	    \frac{1}{n}\left(V^A_n - V^B_n\right)\\
	    (\mu_A - \mu_B) - (\mu_A - \mu_B) <&
	    \frac{1}{n}\left(V^A_n - V^B_n\right)\\
	    0 <& \frac{1}{n}\left(V^A_n - V^B_n\right) \\
	\end{align*}
	Thus, $V^A_n > V^B_n$.
\end{proof}
We can show a similar result for the specialized bias-optimal case of average reward. 
\begin{definition}[Bias Value]
The relative difference in total reward gathered is
\begin{align*}
	V^\pi \triangleq \lim_{n\rightarrow \infty} \Ebf\left[\sum_{i=1}^n( R^\pi_i - \mu_\pi)\right].
\end{align*}
\end{definition}
\begin{restatable}{proposition}{biasoptimal}
	For any policies $\pi_A,\pi_B$ whose average rewards $\mu_A,\mu_B$ exist, if $\mu_A = \mu_B$ and the bias values exist with $V^A > V^B$, then there exists an $N$ such that $V^A_n > V^B_n$ for all $n\geq N$.
\end{restatable}
\begin{proof}	
	Assume $\mu_A = \mu_B$ exist as well as the bias optimal values $V^A > V^B$.  
	\begin{align*}
		V^A &\triangleq \lim_{n\rightarrow \infty} \Ebf\left[\sum_{i=1}^n( R^A_i - \mu_A)\right], & V^B &\triangleq\lim_{n\rightarrow \infty} \Ebf\left[\sum_{i=1}^n( R^B_i - \mu_B)\right].
	\end{align*}
	Distributing the expectation, breaking the summation, and recognizing the expected $n$-step sums as $V_n^{A}=\Ebf[\sum_{i=1}^n R^A_i]$ and $V_n^{B}=\Ebf[\sum_{i=1}^nR^B_i]$, we have
	\begin{align*}
		V^A &= \lim_{n\rightarrow \infty}(V^A_n - n\mu_A), & V^B &\triangleq\lim_{n\rightarrow \infty} ( V^B_n - n\mu_B).
	\end{align*}	

	According to the definition of a limit, for any $\varepsilon_A,\varepsilon_B>0$, there exists some $N_A$ and $N_B$ such that 
	\begin{align*}
		V^A-\varepsilon_A < &V_k^{A}-k\mu_A & 		
		V^B+\varepsilon_B > &V_m^{B}-m\mu_B,
	\end{align*}
	for all $k\geq N_A$ and $m \geq N_B$. 
	
	Choose $\varepsilon_A=\varepsilon_B=\frac{1}{2}(V^A - V^B) > 0$ and let
	$N=\max\{N_A,N_B\}$, so that for all $n\geq N$, the above inequalities hold.  If you negate the second inequality and add them together, then for all $ n > N$,
	\begin{align*}
	    (V^A - \varepsilon_A) - (V^B + \varepsilon_B) <&
	    \frac{1}{n}V^A_n - \frac{1}{n}V^B_n - n\mu_A + n\mu_B\\
	    \intertext{Since $\mu_A = \mu_B$,}
	    (V^A - \varepsilon_A) - (V^B + \varepsilon_B) <&
	    \frac{1}{n}\left(V^A_n - V^B_n\right)\\
	    (V^A - V^B) - (\varepsilon_A + \varepsilon_B) <&
	    \frac{1}{n}\left(V^A_n - V^B_n\right)\\
	    (V^A - V^B) - (V^A - V^B) <&
	    \frac{1}{n}\left(V^A_n - V^B_n\right)\\
	    0 <& \frac{1}{n}\left(V^A_n - V^B_n\right) \\
	\end{align*}
	Thus, $V^A_n > V^B_n$.
\end{proof}
\section{A Constructive Algorithm}

We now develop an algorithm that can construct the realizing reward function given a preference relation that is known to satisfy Axioms 1-5. 
\autoref{alg:reward_design} uses the preference relation to sort outcomes, then computes rewards and two-step utilities by scaling their rankings relative to their break-even point with the best and worst outcomes. 
With this information, the discount factor can be computed in closed-form.
Below we summarize the procedures for using a preference relation to sort and scale outcomes. 

PrefSort (\autoref{alg:pref_sort}) is a procedure for sorting a set of outcomes according to their preference.
Our implementation takes in a preference relation and set of outcomes $\Xcal$.
The procedure returns a tuple of outcomes which are sorted in ascending order, according to $\Rcal$, with MergeSort.  

PrefScale (\autoref{alg:pref_scale}) is a procedure to determine the relative degree of preference between outcomes.
In our implementation, it takes a preference relation, a tuple of preference-sorted outcomes $\Xcal$, and a tolerance parameter $\varepsilon\in(0,1]$.
The procedure returns a set of numerical scale factors that reflect the degree to which each outcome is preferred relative to best and worst outcomes. 
Our implementation assigns scale factors using a binary line search informed by the continuity axiom \autoref{axiom:continuity}.
The inner loop of the line search terminates when the difference between subsequent factors differ less than a pre-specified $\varepsilon$.

The complexity of \autoref{alg:reward_design} only depends on $|\A|$ and $|\O|$.
The call to PrefSort requires $O(n \log n)$ operations, where $n = 2|\A \times \O|$, ignoring an additive constant. 
The call to PrefScale requires $O(n)$ operations.
We then run two for loops, the largest of which iterating through $|\Ucal| = 2|\A \times \O| + 2$ elements. 
Thus, the total run-time is $O\left(2|\A \times \O| \log|\A \times \O| \right)$.

A Python implementation of \autoref{alg:reward_design} is available at \url{https://github.com/jdmartin86/reward-hypothesis}.

\begin{algorithm}[!t]
\caption{Reward and Discount Design}
\label{alg:reward_design}
\textsc{Input:} $\Rcal=\{\preflt, \prefle, \prefgt, \prefge, \sim \}$. \\
\textsc{Output:} $r : \Acal \times \Ocal \rightarrow \mathbb{R}$, $\gamma : \Acal \times \Ocal \rightarrow [0,1]$. 

\begin{algorithmic}[1]
\STATE $\Xcal_1 = \Acal \times \Ocal\cup \{\varepsilon\}$
\STATE $\Xcal_2 = \{ t \cdot t \ \colon \ t \in \Xcal_1\}$ 

\STATE $\Ucal = $ PrefSort$(\Rcal, \Xcal_1\cup\Xcal_2 )$
\STATE $\Pcal = $ PrefScale$(\Rcal, \Ucal,\epsilon)$
\STATE Let $i_\varepsilon$ be the index of $\varepsilon$ in $\Ucal$.
\FOR{$i \in \{1,\cdots,|\Ucal|\}$}
	\STATE $u(\tau_i) = p_{i_\varepsilon}-p_i$
\ENDFOR

\FOR{$i\in\{1,\cdots,|\Xcal_1|\}$}
    \STATE $r(t_i) = u(t_i)$
	\STATE $\gamma(t_i) = u(t_i\cdot t_i)/u(t_{i}) -1$
\ENDFOR

\STATE \textbf{return} $r$, $\gamma$
\end{algorithmic}
\end{algorithm}

\begin{algorithm}[!t]
\caption{PrefSort}
\label{alg:pref_sort}
\textsc{Input:} $\Rcal, \Xcal=(t_1,\cdots,t_n)$. \\
\textsc{Output:} Sorted $\Xcal$. 

\begin{algorithmic}[1]
\IF{$n\leq 1$}
	\STATE \textbf{output} $\Xcal$
\ENDIF
\STATE $\Lcal, \Kcal \gets ()$
\FOR{$i =1,\cdots,n$}
	\IF{$t_i \prefge t_{\lfloor n/2\rfloor}$}
	\STATE $\Kcal \gets \text{concat}(\Kcal, t_i)$
	\ELSE
	\STATE $\Lcal \gets \text{concat}(\Lcal, t_i)$
	\ENDIF
\ENDFOR
\STATE $\Lcal\gets$PrefSort$(\Lcal)$
\STATE $\Kcal\gets$PrefSort$(\Kcal)$	
\STATE $\Xcal \gets \text{concat}(\Lcal, \Kcal)$
\STATE \textbf{output} $\Xcal$
\end{algorithmic}
\end{algorithm}

\begin{algorithm}[!t]
\caption{PrefScale}
\label{alg:pref_scale}
\textsc{Input:} $\Rcal, \Xcal=(t_1 \prefle t_2 \prefle \cdots), \epsilon\in(0,1]$. \\
\textsc{Output:} $\Pcal = \{p_1,\cdots,p_{|\Xcal|}\}$. 

\begin{algorithmic}[1]
\FOR{$i\in\{1,\cdots,|\Xcal|\}$}	
\STATE $p^{k-1}_i\gets 1$
\STATE $\Delta \gets 2\epsilon$
\WHILE{$\Delta \geq \epsilon$}
	\IF{$t_i \succ p^{k-1}_it_1 + (1-p^{k-1}_i)t_{|\Xcal|}$}
	\STATE $p^k_i \gets 3p^{k-1}_i / 2$
	\ELSIF{$t_i \prec p^{k-1}_it_1 + (1-p^{k-1}_i)t_{|\Xcal|}$}
	\STATE $p^k_i \gets p^{k-1}_i/2$
	\ELSE
	\STATE \textbf{break}
	\ENDIF
	\STATE $\Delta\gets |p^{k}_i-p^{k-1}_i|$
	\STATE $p^{k-1}_i\gets p^{k}_i$
\ENDWHILE
\STATE $p_i= p^{k-1}_i$
\ENDFOR
\STATE \textbf{output} $\Pcal$.
\end{algorithmic}
\end{algorithm}

\section{Additional Comments on Objective Goals}
As a special case, consider when the environment can be modeled as a Partially Observable MDP (POMDP, \cite{cassandra1994acting}), and suppose that the designer observes the environment states and the agent's actions, $\bar{\Ocal}=\Scal\times\Acal$, so that histories are $\bar{h}=a_1,s_1,a_2,s_2,\ldots$.
If Axioms 1-5 apply to the designer's preference relation on this distributions over histories, where the reference to transition $(\O \times \A)$ in Axiom 5 now refers to $(\S \times \A)$, then \autoref{thm:markov-reward} extends to reward functions $r : (\S\times \A) \rightarrow \mathbb{R}$ and discount functions $\gamma : (\S\times \A) \rightarrow [0,1]$.

Note that the in the objective goals setting, the designer's states may encode more than necessary to produce Markov state transitions.  
For example, they could encode a reward bundle~\citep{abelexpressing}, a finite state machine defining an otherwise non-Markov reward.  This should allow us to define a variant of \autoref{axiom:temporal_gamma_indifference} that only needs to be satisfied when state transitions consist of POMDP states produced with any other finite state machine that transitions on POMDP transitions.  In other words, preferences that can be expressed as reward bundles are captured by this extended axiom.

\section{Additional Comments on Multiple Objectives}
Here we examine the temporal nature of handling multiple objectives in view of an agent's lived experience. 
We present an additional axiom with which we may want to constrain goals and purposes.  
Let $A[h\rightarrow B]$ refer to the distribution over histories where all histories with non-zero support in $A$ that share the prefix $h$ are removed, and their support shifted to $h\cdot B$.  

\begin{axiom}[Sequential Consistency]\label{axiom:seq-consistency}
For all $A, B, C \in\Delta(\hists)$ and $h \in \hists$, $h \cdot A \prefgt h \cdot B$ if and only if $C[h\rightarrow A] \prefgt C[h\rightarrow B]$.
\end{axiom}

In other words, extensions to a hypothetical history that are more aligned with some goal or purpose do not change if that hypothetical history becomes certain.  
In behavioral terms, goal-aligning behavior following a possible future is the same as goal-aligning behavior following that future occurring.  
Alternatively, goal-aligning behavior after some history should not depend on hypothetical alternative pasts that did not come about.

This is related to notions of \emph{dynamic inconsistency} in behavioural economics\footnote{It is best to think of ``time passing'' in our formalization as ``the past'' at some time $t$ becoming certain --- fixing $h_t \in \hists_t$ --- and so the policy and environment specify distributions over histories whose support has $h_t$ as a prefix.}: one may prefer to receive \$110 in 101 days to \$100 in 100 days, and yet when 100 days passes the same person may now prefer \$100 to waiting one day for \$110 (i.e., human preferences are often not dynamically consistent).  It is straightforward to see that Independence (\autoref{axiom:independence}) implies Sequential Consistency (Axiom~\ref{axiom:seq-consistency}).  However, while one may reject Independence, it seems much more difficult to reject Sequential Consistency, i.e., the goal-alignment of extensions of histories change if the history becomes certain.  Notice, however, that Constrained MDP formulations for capturing multiple objectives, in fact, violate this more specific axiom.

\section{Additional Related Work}
\subsection{Economics}
Economics has been studying the nature of rational behavior for centuries. In the 1700s, Gabriel Cramer and Daniel Bernoulli independently formulated what is now called the Expected Utility Hypothesis \citep{machina1990expected} in response to the ``St. Petersburg Paradox'' \citep{martin2011st} articulated by Daniel's cousin Nicholas \cite{bernoulli1738exposition}. The Expected Utility Hypothesis says, roughly, that individuals ``might maximize the expectation of ‘utility’ rather than of monetary value.'' 
\citep{machina1990expected}. 
Centuries later, \citet{ramsey2016truth} provided the first formal axiomatic treatment of expected utility, which would later be refined by \citet{vonneumann1953theory} to form the now widely adopted foundations of decision theory. Following this development, an expansive body of research has explored how to account for other aspects of rationality, including uncertainty \citep{kreps1978temporal}, time \citep{koopmans1960stationary,koopmans1964stationary}, and computation \citep{lewis1985effectively,rustem1990rationality,richter1999computable}.

\section{Constant Discounting}
In this section, we comment on the constant discounting case, which is commonly employed in practice. 
According to our results, if the preference relation satisfies Temporal $\gamma$-Indifference with a constant discount, then there exists a Markov reward that can effectively express the desired goal.
However, it's important to note that \citet{pitis2019rethinking} highlights the limitation of using a constant discount to express general preferences. 
This observation holds true, particularly in episodic problems where the discount is set to zero upon termination and remains at a constant positive value less than one elsewhere. 
While our results encompass scenarios where the discount remains constant and is applied exponentially, we do not specifically address situations where the discount is applied in a hyperbolic manner, as discussed by \citet{fedus2019hyperbolic}.

\end{document}